\documentclass[letterpaper, 10 pt, conference]{ieeeconf}  

\IEEEoverridecommandlockouts                              

\usepackage{stfloats}
\usepackage{amsmath,amssymb}

\usepackage{amsthm}
\usepackage{graphicx}
\usepackage{booktabs}
\usepackage{xcolor}
\usepackage{hyperref}
\usepackage{tikz}
\usepackage{algorithm}
\usepackage{algorithmic}
\usepackage{multirow}
\usepackage{subcaption}
\usepackage{pifont}
\usepackage[inkscapelatex=false]{svg}
\newcommand{\cmark}{\ding{51}}
\newcommand{\xmark}{\ensuremath{-}}

\newtheorem{proposition}{Proposition}
\newtheorem{remark}{Remark}
\usepackage{cite}
\usepackage{svg}
\usepackage{hyperref}

\graphicspath{{results_publication_n50/figures/}}
\title{\LARGE \bf Efficient Hybrid Hard--Soft Signal Temporal Logic MPC for Driving
  under Multimodal Uncertainty}
  
\author{Kumar Manas and Adrian Paschke%
\thanks{K. Manas is with Technical University Darmstadt, Germany.
This work was conducted while he was affiliated with
Freie Universit{\"a}t Berlin, Germany.
Email: \texttt{first.last@tu-darmstadt.de}.\newline
A. Paschke is with Freie Universit{\"a}t Berlin
and Fraunhofer FOKUS, Berlin, Germany.}%
}

\begin{document}
\maketitle

\begin{abstract}
Motion planning for automated vehicles in interactive traffic has to simultaneously ensure safety, respect traffic rules, and maintain comfortable progress under multi-modal behavior uncertainty. Imposing all these requirements as hard constraints reduces feasibility, whereas treating them all as costs allows safety-critical requirements to be traded against progress. We propose a hybrid hard--soft Signal Temporal Logic (STL) MPC framework that optimizes a single shared ego trajectory against every retained behavior hypothesis: minimum-distance and speed predicates are imposed as explicit hard constraints, while soft STL robustness terms shape clearance, comfort, and interaction quality within the resulting feasible driving corridor. This retains robustness to discrete interaction modes without branch-specific contingent ego policies. We further identify substantial symbolic graph-construction overhead and introduce reachability-certified relevance screening and vectorized symbolic construction to reduce computational cost while preserving the feasible set and objective. Closed-loop Monte Carlo evaluation over six highway and intersection scenarios shows the highest strict success rate among the compared planners with no observed collisions or rule violations, at roughly a quarter of the planning latency of the unoptimized formulation. An ablation study separates the two layers: the hard constraints define the safety envelope, while the soft terms improve clearance and motion quality within it.
Project code:
\href{https://github.com/kumarmanas/Hybrid_STL_MPC_Planner}
{\nolinkurl{Hybrid_STL_MPC_Planner}}.
\end{abstract}


\section{Introduction}

Autonomous driving in interactive traffic requires motion planners that jointly
ensure (i)~\emph{safety}, by avoiding collisions and maintaining sufficient
clearance; (ii)~\emph{traffic-rule compliance}, by respecting constraints such
as speed limits, intersection-yield requirements, and no-passing restrictions; and
(iii)~\emph{comfort and efficiency}, by generating smooth trajectories that
make progress without unnecessary conservatism. These objectives become
difficult to reconcile under multi-modal behavior uncertainty: an interacting
vehicle may yield or proceed, while planning only against the most likely mode
can become unsafe and worst-case reasoning can sacrifice progress
unnecessarily~\cite{chen2022branchmpc,mustafa2025racp}.

Model Predictive Control (MPC) and Model Predictive Contouring Control (MPCC) provide a principled framework for combining vehicle dynamics, actuator limits, safety constraints, and performance objectives~\cite{brito_model_2019}. Uncertainty has been addressed through risk-aware, barrier-function, robust, and scenario-based formulations~\cite{hakobyan2019cvar, batkovic2021robustscenario}. Contingency and Branch MPC further represent discrete behavior hypotheses, with one predicted trajectory for each candidate intent, such as the yield-or-proceed case above, and enable delayed commitment through mode-dependent future policies~\cite{alsterda2019cmpc, chen2022branchmpc, zhang2024riskawarebmpc}. However, their branch-specific decision variables increase the complexity of the optimization.

In contrast, we consider a \emph{shared-trajectory multi-hypothesis} formulation:
a single receding-horizon ego trajectory is optimized against all retained
behavior hypotheses. This preserves robustness to discrete interaction modes
without introducing branch-specific ego controls or non-anticipativity
constraints.

Signal Temporal Logic (STL) provides an expressive language for safety and
traffic-rule specifications with quantitative robustness semantics~\cite{raman2014model,mehdipour2023formalroads}. STL has been integrated with
MPC through robustness-based and optimization-based formulations~\cite{farahani2015robuststl,lin2023mpr,takayama2025stlccp}; however, applying
all rules either as strict constraints or as relaxable costs is problematic.
Strict formulations may reduce feasibility, while purely soft formulations
allow safety-critical requirements to be traded against progress. Moreover,
evaluating every STL interaction term across all agents and prediction steps
creates substantial computational overhead.

\textbf{We therefore propose an efficient hybrid hard--soft STL MPC framework
for multi-modal interaction uncertainty.}
Safety-critical minimum-distance and speed-limit predicates are imposed as
explicit hard constraints, while soft STL robustness terms encourage
larger clearances, comfortable motion, and rule-compliant interaction. Thus,
the hard layer defines the safety envelope, while the soft layer shapes
the behavior within the safety envelope.

To reduce computational cost without changing the represented optimization
problem, we introduce reachability-based relevance screening for short-range STL penalties and vectorized symbolic graph construction.
The closed-loop framework further combines safety-gated decision postponement, warm-starting, and a fallback strategy that removes soft shaping terms while
retaining the hard safety constraints.

The main contributions are:
\begin{enumerate}
\item A hybrid hard--soft STL formulation for shared-trajectory
    multi-hypothesis MPC, in which non-relaxable safety predicates and
    preference-level traffic rules occupy distinct roles within a single
    nonlinear program, avoiding branch-specific contingent ego policies.

    \item Reachability-certified relevance screening and vectorized
  symbolic construction that reduce computational overhead.
  The screening certificate permits redundant hard constraints
  and identically zero soft penalties to be omitted \emph{before}
  construction, preserving the feasible set and objective.

    \item A closed-loop execution strategy in which decision postponement is
    admitted only when a safety shield certifies the hard clearance margin,
    and solver failure is handled by a fallback that removes soft shaping
    while retaining the hard safety constraints.
\end{enumerate}
\vspace{-7pt}
\section{Related Work}
\label{sec:related}

\begin{table*}[!b]
\centering
\caption{Comparison with representative uncertainty-aware and STL-MPC methods.
``Multi-modal'' indicates discrete behavior hypotheses; ``hard/soft split''
distinguishes non-relaxable constraints from preferences.}
\label{tab:positioning}
\scriptsize
\renewcommand{\arraystretch}{1.12}
\setlength{\tabcolsep}{3.5pt}
\begin{tabular}{p{3.25cm}ccccc p{3.2cm}}
    \toprule
    Method
    & Multi-modal
    & Ego policy
    & STL
    & Hard/soft split
    & Postpone / recovery
    & Computational mechanism \\
    \midrule

    Robust scenario MPC~\cite{batkovic2021robustscenario}
    & \cmark
    & mode-dependent
    & \xmark
    & \xmark
    & \xmark
    & tube--scenario formulation \\

    Contingency MPC~\cite{alsterda2019cmpc}
    & partial
    & parallel
    & \xmark
    & \xmark
    & contingency
    & parallel prediction horizons \\

    Branch MPC~\cite{chen2022branchmpc}
    & \cmark
    & tree
    & \xmark
    & \xmark
    & implicit
    & scenario-tree structure \\

    Learning-based BMPC~\cite{bouzidi2024learningbmpc}
    & \cmark
    & tree
    & \xmark
    & \xmark
    & \cmark
    & online scenario selection \\

    Risk-aware BMPC~\cite{zhang2024riskawarebmpc}
    & \cmark
    & tree
    & \xmark
    & \xmark
    & \xmark
    & iterative min--max risk update \\

    Game-theoretic BMPC~\cite{zhang2023bmpcmerge}
    & \cmark
    & tree
    & \xmark
    & \xmark
    & \xmark
    & tailored tree solver \\

    STL-MPC~\cite{raman2014model,farahani2015robuststl}
    & \xmark
    & single
    & \cmark
    & \xmark
    & \xmark
    & mixed-integer encoding \\

    Low-complexity STL-MPC~\cite{cep2024stllowcomplex}
    & \xmark
    & single
    & \cmark
    & \xmark
    & \xmark
    & formula simplification / input blocking \\

    Priority-softened MPC~\cite{quan2025priority}
    & \xmark
    & single
    & \xmark
    & \cmark
    & softening
    & learned constraint selection \\

    \textbf{Hybrid STL-MPC (Ours)}
    & \cmark
    & \textbf{shared}
    & \cmark
    & \cmark
    & \cmark
    & \textbf{exact screening + vectorization} \\

    \bottomrule
\end{tabular}
\vspace{-10pt}
\end{table*}

Table~\ref{tab:positioning} summarizes our method against representative prior
work; the key distinctions are discussed below.

\textbf{Multi-modal and contingency MPC.}
Risk-aware and robust MPC address uncertain interactions through risk-sensitive
or scenario-based formulations~\cite{hakobyan2019cvar,batkovic2021robustscenario}.
Contingency and Branch MPC explicitly represent alternative future behaviors and
optimize consistent or mode-dependent ego policies across them
~\cite{alsterda2019cmpc,chen2022branchmpc,zhang2024riskawarebmpc,
zhang2023bmpcmerge,bouzidi2024learningbmpc}.
In contrast, our formulation retains multiple behavior hypotheses while
optimizing a \emph{single shared ego trajectory}, avoiding branch-specific ego
controls and non-anticipativity constraints.

\textbf{STL in MPC.}
STL has been integrated with MPC through mixed-integer encodings, robustness
optimization, and more scalable approximations
~\cite{raman2014model,farahani2015robuststl,lin2023mpr,
cep2024stllowcomplex,takayama2025stlccp}.
Recent work further considers uncertain uncontrollable agents through
data-driven uncertainty quantification~\cite{aut2026conformalstl}.
Our focus is complementary: safety-critical predicates are enforced as explicit
constraints, whereas soft STL robustness terms shape clearance, comfort, and
interaction quality.

\textbf{Prioritization and computational efficiency.}
Priority-based softening resolves conflicts among adjustable constraints while preserving hard safety requirements~\cite{quan2025priority}. Constraint-removal methods accelerate nonlinear MPC by certifying inactive constraints from offline cost-function bounds~\cite{dyrska2021accelerating}. Our screening differs in certificate and timing: relevance is determined geometrically from the ego reachable set and STL penalty support, recomputed each step, and irrelevant terms are removed \emph{before} symbolic NLP construction. \textit{This reduces instantiated terms, graph complexity, derivative evaluation, and solve cost.} Vectorized construction further reduces graph size without changing the objective or feasible set, while decision postponement and hard-preserving fallback support efficient closed-loop execution.

\vspace{-2pt}
\section{Problem Formulation and Preliminaries}
\label{sec:problem}

\begin{figure*}[t]
\centering
\includegraphics[width=\textwidth]{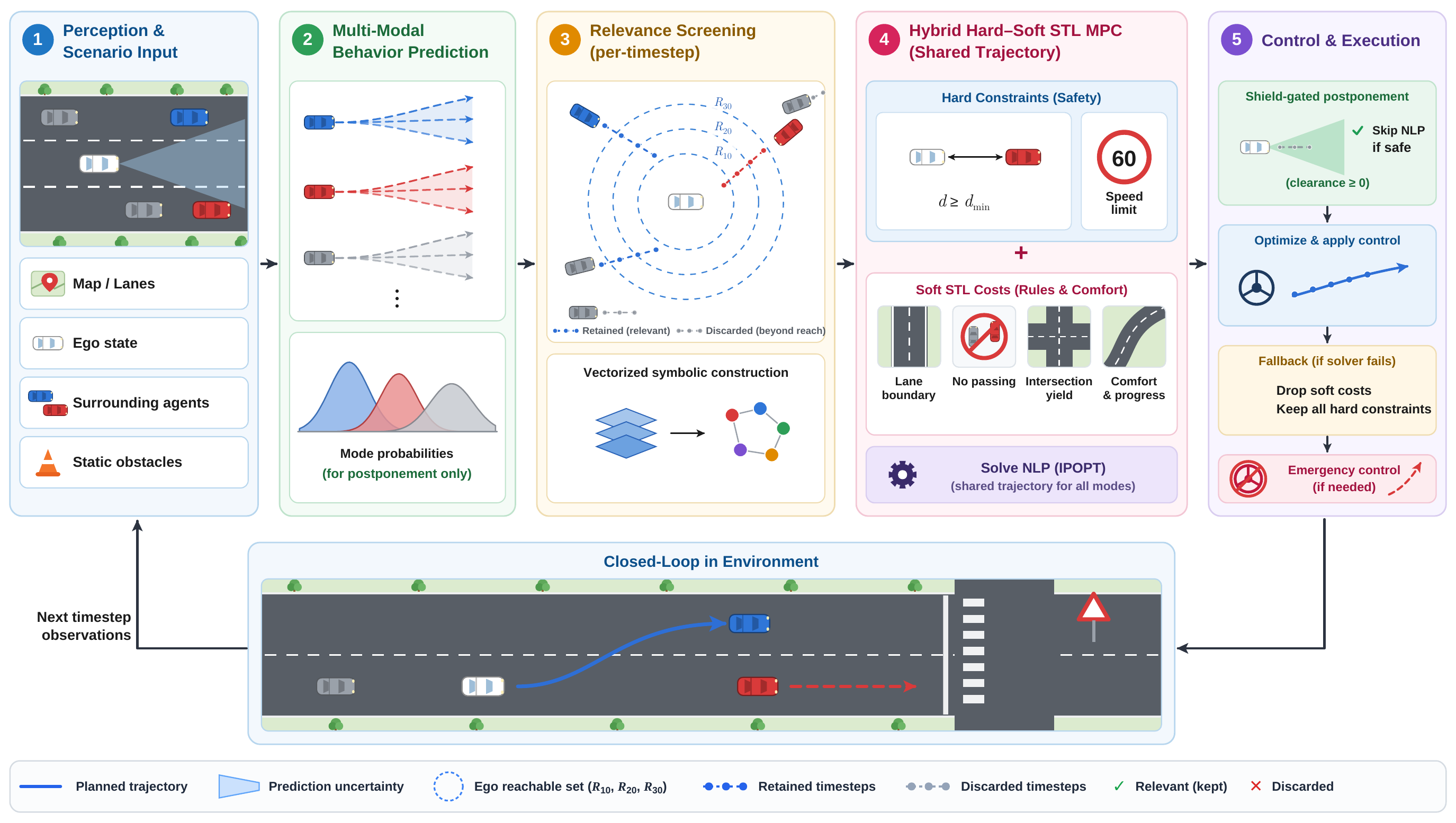}
\caption{Proposed Hybrid STL-MPC pipeline. Multi-modal predictions constrain a
\emph{single} shared ego trajectory; no scenario tree or branch-specific
controls are built, and mode probabilities enter only the postponement trigger.
Minimum distance and speed limit are explicit constraints;
  soft costs shape preferred clearance, speed, and scenario-specific
  intersection-yield and no-passing behavior. Screening is
per (participant, timestep): terms beyond the ego reachable bound $R_t$ are
provably inactive and removed before symbolic construction, leaving the
feasible set and objective unchanged (Proposition~\ref{prop:screening}).
Postponement is evaluated before each solve and bypasses it when a shield
certifies non-negative clearance; solver failure triggers a fallback that drops
soft costs but retains every hard constraint.}
\label{fig:framework}
\vspace{-15pt}
\end{figure*}

Fig.~\ref{fig:framework} summarizes our planning pipeline. At every replanning
step, the planner receives the current ego state, static obstacles, and
multi-modal predictions of surrounding agents. In contrast to Branch MPC, we
optimize a \emph{single shared ego trajectory} against all supplied behavior
hypotheses. The main design principle is to separate safety-critical feasibility
from preference-level STL shaping while keeping the resulting nonlinear program
tractable for receding-horizon execution.

\subsection{Ego dynamics and multi-modal predictions}

The ego state is
$x=[p_x,p_y,\psi,v]^\top\in\mathbb{R}^4$ with control
$u=[a,\delta]^\top\in\mathbb{R}^2$. We use the kinematic bicycle model
\begin{equation}
\dot p_x=v\cos\psi,\quad
\dot p_y=v\sin\psi,\quad
\dot\psi=\frac{v\tan\delta}{L_w},\quad
\dot v=a,
\label{eq:dynamics}
\end{equation}
which is discretized by forward Euler,
$x_{t+1}=f(x_t,u_t)$, with sampling time $\Delta t$.
The decision variables over horizon $H$ are
$X=\{x_0,\ldots,x_H\}$ and
$U=\{u_0,\ldots,u_{H-1}\}$, subject to state, control, and jerk bounds. Because the same ego state and control sequence is optimized against all
retained behavior hypotheses, no branch-specific controls or
non-anticipativity constraints are required.
Our implementation uses $\Delta t=0.1$\,s, horizon $H=30$, and wheelbase $L_w=2.578$\,m.

Let $\mathcal{A}=\{1,\ldots,N\}$ denote the dynamic agents. For agent $i$,
the predictor supplies a finite set of behavior hypotheses
\begin{equation}
\mathcal{H}_i =
\left\{
\tau_i^{(m)}
=
\{x_{i,t}^{(m)}\}_{t=0}^{H}
\right\}_{m\in\mathcal{M}_i}.
\label{eq:modes}
\end{equation}

Here $\mathcal{M}_i$ indexes the hypotheses of agent $i$, $\mathcal{H}=\bigcup_{i\in\mathcal{A}}\mathcal{H}_i$, and $\mathcal{O}$ collects the static obstacles. The trajectories are time-aligned at every replanning step. Predictor
probabilities may additionally be available, but the primary MPC does not
probability-weight the modes: every supplied hypothesis constrains the same ego
trajectory. Probabilities are used only by the decision-postponement mechanism
described in Sec.~\ref{sec:method}. This avoids enumeration of joint mode
combinations and branch-specific ego controls.

\subsection{Hybrid hard--soft STL representation}

We represent traffic requirements through signed margins
$g_r(x_t,\mathcal{H},\mathcal{O})$, where positive values indicate satisfaction.
Rather than assigning every rule the same optimization role, we divide the
active specification into hard safety predicates
$\mathcal{R}_{\mathrm{h}}$ and soft preference predicates
$\mathcal{R}_{\mathrm{s}}$.

For $r\in\mathcal{R}_{\mathrm{h}}$, the corresponding STL specification is
\begin{equation}
\varphi_r^{\mathrm{h}}
=
\mathbf{G}_{[0,H]}
\left(g_r(x_t,\cdot)\geq 0\right),
\label{eq:hard_stl}
\end{equation}
and is imposed directly as an optimization constraint. For a soft predicate
$r\in\mathcal{R}_{\mathrm{s}}$, we use the quadratic deficit
\begin{equation}
\ell_r(t)
=
w_r
\left[-g_r(x_t,\cdot)\right]_+^2,
\qquad
[z]_+ = \max(z,0),
\label{eq:soft_stl}
\end{equation}
with per-rule weight $w_r>0$, which is zero once the preferred margin is
satisfied. For distance-type predicates the support is a bounded neighborhood
of the participant, later enabling exact relevance screening.
 We populate the STL rule library automatically using TR2MTL~\cite{Manas_2024};
  Table~\ref{tab:stl_params} details the hard and soft roles
  of the benchmark rules.

\begin{table}[t]
\centering
\caption{STL terms entering the reported benchmark. ``Hard'' denotes an explicit
NLP constraint; soft terms enter the quadratic objective. Distance thresholds
are adapted to scenario geometry as described in Sec.~\ref{sec:experiments}.}
\label{tab:stl_params}
\scriptsize
\setlength{\tabcolsep}{3pt}
\begin{tabular}{p{1.55cm}p{2.4cm}p{3.25cm}}
\toprule
\textbf{Rule} & \textbf{Hard role} & \textbf{Soft role} \\
\midrule
Distance
& Explicit minimum clearance against all retained hypotheses
& Preferred clearance $d_{\mathrm{safe}}$ \\

Speed
& Explicit upper-speed constraint
& Preferred upper speed threshold \\

Intersection yield
& Covered by generic hard distance safety
& Gated crossing-traffic clearance cost \\

No passing
& --
&  Zone-gated soft cost plus a separate high-weight lateral-deviation penalty \\
\bottomrule
\end{tabular}
\vspace{-15pt}
\end{table}

\paragraph{Distance safety and preferred clearance.}
For a participant hypothesis $(i,m)$ with position $p_{i,t}^{(m)}$ and radius
$r_i$, the hard clearance condition is
\vspace{-5pt}
\begin{equation}
\left\|p_t-p_{i,t}^{(m)}\right\|_2
\geq d_{\min}+r_i+b,
\qquad
\forall i,m,t,
\label{eq:hard_distance}
\end{equation}
where $d_{\min}=1.5$\,m, vehicles use $r_i=2.5$\,m, and $b=0.3$\,m is a
numerical safety buffer. Static
obstacles are treated analogously. Evaluation uses the unbuffered threshold
$d_{\min}$. The corresponding soft term replaces $d_{\min}+b$ by the preferred
clearance $d_{\mathrm{safe}}=6$\,m and applies~\eqref{eq:soft_stl}. Thus the hard layer defines the collision-avoidance envelope, while the soft layer encourages
additional clearance within the feasible region.
\vspace{-2pt}
\paragraph{Speed and interaction preferences.}
The posted speed limit is imposed explicitly as
$v_t\leq v_{\max}^{\mathrm{hard}}=25$\,m/s, with a soft preference at
$v_{\max}^{\mathrm{soft}}=20$\,m/s. At intersections, crossing agents are
identified through heading difference and proximity to the conflict region.
Their yield-specific term is a soft clearance penalty gated by heading and
conflict distance; collision avoidance remains protected by
Eq.~\eqref{eq:hard_distance} for every predicted mode. In no-passing scenarios,
lateral deviation from the initial lane is penalized by a high-weight quadratic
cost activated through a smooth spatial gate over the restricted zone.

\subsection{Shared-trajectory Hybrid STL-MPC}

At planning step $k$, the proposed planner solves one nonlinear program over
the shared ego trajectory:
\begin{align}
\min_{X,U}\quad
& J_{\mathrm{prog}}(X)
+ J_{\mathrm{ctrl}}(U)
+ J_{\mathrm{jerk}}(U) \nonumber\\[-2pt]
&\quad
+ \lambda_{\mathrm{stl}}J_{\mathrm{soft}}(X,\mathcal{H})
+ J_{\mathrm{np}}(X)
\label{eq:opt}\\
\text{s.t.}\quad
& x_0=x_k,\qquad
x_{t+1}=f(x_t,u_t), \nonumber\\
& x_t\in\mathcal{X},\qquad
u_t\in\mathcal{U}, \nonumber\\
& \left\|p_t-p_{i,t}^{(m)}\right\|_2
\geq d_{\min}+r_i+b,
\quad \forall i,m,t, \nonumber\\
& \left\|p_t-p_o\right\|_2
\geq d_{\min}+r_o+b,
\quad \forall o,t, \nonumber\\
& v_t\leq v_{\max}^{\mathrm{hard}},\qquad
|j_t|\leq j_{\max}.
\nonumber
\end{align}
Here \( J_{\mathrm{soft}}
=
\sum_{r\in\mathcal{R}_{\mathrm{s}}}
\sum_{t=0}^{H-1}\ell_r(t) \)
collects the quadratic STL preference terms with scaling $\lambda_{\mathrm{stl}}$. In no-passing scenarios, the restricted-zone rule additionally contributes $J_{\mathrm{np}}$, a high-weight quadratic lateral-deviation penalty gated smoothly over the zone, applied outside that scaling, and retained whenever $J_{\mathrm{soft}}$ is
dropped. The progress objective combines a
weak running position cost and a stronger terminal cost,
\begingroup
\setlength{\abovedisplayskip}{3pt}
\setlength{\abovedisplayshortskip}{3pt}
\begin{equation}
J_{\mathrm{prog}}
=
w_R\sum_{t=0}^{H-1}\|p_t-p_g\|_2^2
+
w_T\|p_H-p_g\|_2^2,
\end{equation}
\endgroup
with $w_R=0.1w_p$ and $w_T=5w_p$. Control effort penalizes acceleration and steering magnitude.

To maintain control continuity, jerk is penalized within the horizon and across
successive MPC calls:
\begin{equation}
j_t=\frac{a_t-a_{t-1}}{\Delta t},
\qquad
J_{\mathrm{jerk}}
=
w_j\sum_t
\left(\frac{j_t}{j_{\max}}\right)^2.
\label{eq:jerk}
\end{equation}
For $t=0$, $a_{-1}$ is the acceleration applied by the previous replanning
step; this inter-replan term is weighted twice as strongly. The same $j_{\max}=6$\,m/s$^3$ bound is additionally imposed as a hard jerk constraint; acceleration and steering carry separate control bounds.

Eq.~\eqref{eq:opt} is therefore \emph{not} a scenario-tree optimization:
all agent modes constrain the same $(X,U)$. The hard-only ablation uses the
identical explicit hard constraint set but removes
$J_{\mathrm{soft}}$, allowing us to isolate the contribution of soft STL
shaping under the same safety envelope.

\section{Methodology: Efficient Hybrid STL-MPC}
\label{sec:method}

This section describes the proposed planner components:
(i) multi-hypothesis planning with a shared ego trajectory,
(ii) exact reachability-based relevance screening,
(iii) vectorized symbolic construction,
(iv) safety-gated decision postponement, and
(v) warm-started optimization with hard-preserving fallback.

\subsection{Shared-trajectory multi-hypothesis planning}
\label{subsec:multi_hyp}

At replanning step $k$, each predicted mode is first aligned to the current
time and truncated or extended to the MPC horizon. Let
$\mathcal{H}_k=\{(i,m)\}$ denote the resulting set of agent--mode hypotheses.
Unlike Branch MPC, we do not construct joint scenario combinations or
branch-specific ego controls. Instead, all hypotheses constrain the same
decision variables $(X,U)$ introduced in Sec.~\ref{sec:problem}.

Consequently, the number of ego decision variables remains $O(H)$, while the
interaction constraints scale with
$O(H\sum_i |\mathcal{M}_i|)$ rather than with the Cartesian product of agent
modes. For the hard distance rule, every retained hypothesis $(i,m)$ imposes
\begin{equation}
    \|p_t-p_{i,t}^{(m)}\|_2
    \geq d_{\min}+r_i+b ,
    \qquad t=0,\ldots,H ,
    \label{eq:multi_hyp_hard}
\end{equation}
where $b$ is the optimization-side safety buffer. Hence a feasible ego
trajectory must remain safe with respect to every supplied mode, while soft
costs select more comfortable behavior inside this common feasible region.

\subsection{Exact relevance screening}
\label{subsec:screening}

Instantiating every participant--timestep STL term creates a large symbolic
graph even when many terms cannot become active. We therefore screen terms
using a conservative upper bound on the ego's reachable path length.

Let $R_t$ denote an upper bound on the distance that the ego can travel from
its current position by prediction step $t$. Under the discrete vehicle model,
we compute
\begingroup
\setlength{\abovedisplayskip}{3pt}
\setlength{\abovedisplayshortskip}{3pt}
\begin{align}
    R_0 &= 0, \\
    R_{t+1} &= R_t +
    \Delta t \,
    \min\!\left(v_{\max},
    v_k+t\Delta t\,a_{\max}\right).
    \label{eq:reachable_bound}
\end{align}
\endgroup
Here $a_{\max}=8$\,m/s$^2$ and $v_{\max}=v_{\max}^{\mathrm{hard}}$. Because the state bounds enforce $v_t\geq0$, this bound is independent of steering and upper-bounds displacement for any admissible control sequence.

Let
$d_{i,t}^{(m)}=\|p_{i,t}^{(m)}-p_k\|_2$
be the distance from the current ego position to a predicted participant.
For a distance term with activation threshold $\rho$, the term is retained only
if
\begin{equation}
    d_{i,t}^{(m)}
    \leq R_t+\rho+\epsilon_{\mathrm{scr}},
    \label{eq:screening}
\end{equation}
where $\epsilon_{\mathrm{scr}}=0.5$\,m is an additional conservative screening
margin. For the hard distance constraint,
$\rho=d_{\min}+r_i+b$; for the soft distance cost,
$\rho=d_{\mathrm{safe}}+r_i$.

\begin{proposition}[Exact screening]
\label{prop:screening}
Let $R_t$ satisfy Eq.~\eqref{eq:reachable_bound}, let
$d_{i,t}^{(m)}=\|p_{i,t}^{(m)}-p_k\|_2$, and suppose
$d_{i,t}^{(m)} > R_t+\rho+\epsilon_{\mathrm{scr}}$. Then for every admissible
control sequence: (i) with $\rho=d_{\min}+r_i+b$ the corresponding hard
constraint~\eqref{eq:hard_distance} is satisfied; and (ii) with
$\rho=d_{\mathrm{safe}}+r_i$ the compact-support penalty~\eqref{eq:soft_stl} is
identically zero. Removing the term therefore leaves both the feasible set and
the objective unchanged.
\end{proposition}

\begin{proof}[Proof sketch]
By construction $R_t$ accumulates $\Delta t\,\min(v_{\max},\cdot)$ and hence
upper-bounds the ego path length, which in turn upper-bounds Euclidean
displacement, $\|p_t-p_k\|_2\leq R_t$, independently of steering. The triangle
inequality then gives
\[
  \|p_t-p_{i,t}^{(m)}\|_2
  \geq d_{i,t}^{(m)}-R_t
  > \rho+\epsilon_{\mathrm{scr}}
  > \rho.
  \]

For (i) this is the hard inequality. For (ii) it gives $g_r\geq0$, and
$[-g_r]_+^2=0$ whenever $g_r\geq0$. The same argument is applicable to static
obstacles. \end{proof}

The same principle applies to the crossing-traffic clearance terms used for
intersection yielding. Screening is exact for the hinge and quadratic penalties
considered here; the reported experiments use quadratic penalties.

\subsection{Vectorized symbolic construction}
\label{subsec:vectorization}

After screening, the remaining constraints and soft costs are constructed
vector-wise rather than through scalar Python loops over time. In particular,
the $H$ kinematic bicycle transitions are expressed as four vector equality
constraints; acceleration and steering bounds as two vector constraints;
the velocity bound as one vector constraint; jerk bounds as two; and hard distance constraints as one vector inequality per participant hypothesis over its active timesteps. Soft distance and intersection-yield costs are similarly evaluated over vectors of active indices.

This transformation preserves the same dynamics, constraints, and objective
values while substantially reducing the number of CasADi graph-construction
operations. The scalar and vectorized formulations were numerically verified
to agree to $10^{-9}$ on randomized inputs. Besides reducing graph-construction
cost, the smaller symbolic graph also reduces derivative-evaluation overhead
inside IPOPT.

\subsection{Safety-gated decision postponement}
\label{sec:postponing}

Many times prediction ambiguity makes an immediate optimization unnecessary
or undesirable. Decision postponement skips the current NLP solve and applies
  a shield-checked control for one sampling interval
  ($\Delta t=0.1$\,s). The planner reassesses the trigger and
  clearance condition at the next step. Postponement is requested
  when any of the following conditions holds:

\textbf{C1 (Prediction ambiguity):}
a multi-modal agent lies within $3d_p$ of the ego, with $d_p=3$\,m, and the
probability gap between its two most likely modes is below
$\Delta\pi_{\mathrm{thresh}}=0.25$.

\textbf{C2 (Imminent interaction):}
the closest agent is within $d_p$, lies within $30^\circ$ of the ego's forward
direction, has a relative heading between $60^\circ$ and $120^\circ$, and is
closing at more than $2$\,m/s.

\textbf{C3 (Dense interaction):}
at least three agents simultaneously lie within $3d_p$.

A trigger does not by itself authorize postponement. Instead, a lightweight
safety shield evaluates a deterministic library of bounded acceleration and
steering commands. Candidate accelerations respect the inter-step jerk bound, while steering candidates span straight, half-limit, and full-limit steering in both directions. Candidates violating the acceleration, steering, or jerk bounds used by the MPC are discarded. Each remaining candidate is rolled out for up to 15 steps against all retained agent hypotheses and static obstacles and ranked primarily by its worst predicted clearance.

Postponement is accepted only if the best candidate maintains non-negative hard
distance margin over this rollout; otherwise, the full MPC is solved. Hence
prediction ambiguity can avoid an unnecessary NLP solve without allowing the
postponement mechanism itself to relax the minimum-clearance requirement.

\subsection{Warm-starting and hard-preserving fallback}
\label{subsec:fallback}

We warmstart IPOPT from the previous MPC solution whenever the observed
state remains close to the previously predicted next state. Specifically, the
cached solution is accepted when the position discrepancy is below $2$\,m and
the heading discrepancy is below $0.5$\,rad. The state and control sequences are
shifted by one step, the terminal state is extrapolated via its final velocity
and heading, and the measured state is imposed at $x_0$. If the cache is
invalid, a zero-control constant-velocity rollout provides the initial guess.

If the primary hybrid NLP fails to solve or fails the numerical hard-constraint
validation, the planner performs one \emph{hard-preserving fallback}: the soft
STL costs are removed, and the problem is re-solved with the same dynamics,
control and jerk limits, hard distance and speed constraints, safety buffer,
and relevance screening. Thus, fallback increases numerical flexibility without weakening the explicit safety envelope. If both optimization attempts fail, a bounded emergency control from the same conservative control library is used as a final execution recovery and recorded for evaluation (see Fig.~\ref{fig:framework}).

\begin{algorithm}[t]
  \caption{Hybrid STL-MPC at planning step $k$}
  \label{alg:planner}
  \begin{algorithmic}[1]
  \REQUIRE State $x_k$, goal $x_g$, obstacles $\mathcal{O}$,
  predictions $\mathcal{H}$, rules $\mathcal{R}$,
  cache $\hat S_{k-1}$, previous acceleration $a_{k-1}$
  \ENSURE Applied control $u_k$

  \STATE Align predictions to step $k$
  \IF{C1, C2, or C3 holds}
      \STATE $(u^{\mathrm{sh}},\rho^{\mathrm{sh}})
      \gets$ safety shield
      \IF{$\rho^{\mathrm{sh}}\geq0$}
          \STATE Update $a_k\gets u^{\mathrm{sh}}_a$;
          \RETURN $u^{\mathrm{sh}}$
      \ENDIF
  \ENDIF

  \STATE Compute $R_{0:H}$ and screen terms using
  \eqref{eq:screening}
  \STATE Initialize from a valid shifted cache or rollout
  \STATE Build vectorized NLP~\eqref{eq:opt}; solve with IPOPT
  \IF{solve or hard-constraint validation fails}
      \STATE Re-solve without $J_{\mathrm{soft}}$,
      retaining all hard constraints
  \ENDIF
  \IF{a hard-valid solution is available}
      \STATE Cache $(X^*,U^*)$; set $u_k\gets U^*_0$
  \ELSE
      \STATE $u_k\gets$ bounded emergency control
  \ENDIF
  \STATE Update $a_k\gets u_{k,a}$; \RETURN $u_k$
  \end{algorithmic}
  \end{algorithm}

Algorithm~\ref{alg:planner} summarizes the resulting receding-horizon loop.

\begin{remark}[Constraint satisfaction]
\label{rem:constraint-satisfaction}
For a planning step at which an NLP solution is accepted, the explicit distance and speed inequalities are satisfied over the current prediction horizon up to the solver and numerical validation tolerances. Exact relevance screening does not weaken these constraints: hard predicates are removed only at time indices unreachable within the horizon, and soft terms only where the penalty is identically zero, so neither the feasible set nor the objective is altered. These statements remain conditional on the vehicle model and the supplied agent predictions, apply only to steps executed from an accepted solution, and do not constitute a probabilistic or recursive-feasibility guarantee for future closed-loop behavior. Steps executed under decision postponement or the final emergency-control path bypass the optimization; both are runtime mechanisms evaluated empirically rather than assigned a formal guarantee.
\end{remark}
\section{Experimental Evaluation}
\label{sec:experiments}

We evaluate the proposed Hybrid STL-MPC in closed-loop Monte Carlo simulation
against uncertainty-aware planning baselines and component ablations. The
evaluation answers three questions: (i) whether explicit hard constraints and
soft STL shaping complement each other, (ii) whether the proposed planner
remains effective under multi-modal interactions, and (iii) whether relevance
screening and vectorized construction reduce computational cost without
changing the represented optimization problem.

\subsection{Experimental Setup}

\paragraph{Scenarios.}
We evaluate 6 interactive driving scenarios:
\emph{highway merge}, \emph{stress-test highway},
\emph{challenging intersection}, \emph{challenging left-turn intersection},
\emph{multi-modal intersection}, and \emph{no-passing}.
Each scenario includes 1–12 surrounding agents and varies initial states, IDM parameters, traffic density, and interaction geometry.
The multi-modal setting additionally provides alternative predicted behavior
trajectories while simulator ground truth is kept separate from the planner's
prediction input. The no-passing scenario activates the spatially gated
restricted-zone rule.

\paragraph{Compared methods.}
We compare seven methods:
\emph{Single-mode goal-tracking MPC}, which plans against only the most likely
prediction mode and includes neither hard STL constraints, soft STL shaping,
nor the postponement shield;
\emph{Risk-Aware BMPC}~\cite{zhang2024riskawarebmpc};
\emph{Game-Theoretic BMPC}~\cite{zhang2023bmpcmerge};
\emph{Soft-only}, which removes explicit NLP distance constraints while retaining shared state/control bounds;
\emph{Hard-only}, which retains the same hard constraints as Hybrid but removes
$J_{\mathrm{soft}}$;
\emph{Hybrid (Ours)}; and
\emph{Hybrid (no screening)}, which is mathematically identical to Hybrid but
disables the exact relevance-screening mechanism. Soft-only and Hard-only represent the two prevailing STL-MPC treatments:
relaxable robustness costs and strict constraint encodings, respectively.
Both are applied under an otherwise identical formulation, so the comparison
isolates the treatment rather than the implementation.

All methods share the same vehicle model and state/control bounds, including the 25\,m/s velocity state bound, with $H=30$, $\Delta t=0.1$\,s, and a 600\,ms IPOPT time limit. The proposed method, its ablations, and the single-mode baseline impose $j_{\max}=6$\,m/s$^3$ both as a hard jerk bound and as the normalizer of the jerk penalty; the risk-aware~\cite{zhang2024riskawarebmpc} and game-theoretic~\cite{zhang2023bmpcmerge} baselines retain their published formulations, in which jerk enters only as a soft cost. The proposed method uses quadratic soft STL penalties and a  0.3\,m optimization-side hard-distance buffer above the evaluated threshold; the buffer applies to the variants that enforce distance as an explicit constraint (Hybrid, Hard-only, and the screening ablation).

\paragraph{Protocol and metrics.}
Each method is evaluated on 50 randomized trials for each of the six scenarios
using the same random seed, giving 300 episodes per method and 2100 episodes in
total. Each episode executes at most 80 closed-loop steps. We define
\emph{strict success} as reaching within 5\,m of the goal, reaching at least
50\% of the target velocity, having no collision or recorded rule violation,
and not aborting due to planning failure. For Hybrid, decision postponement is exercised on 10.5\% of executed steps
overall and on 16.5--22.0\% of steps in the three intersection scenarios; these steps bypass the optimization and are monitored separately. Crossing-traffic yielding is verified against simulator ground truth on every executed step, including those under postponement; no violation was recorded across the 1031 postponed steps in the three intersection scenarios.

Outcome rates (success, violation, collision, and abort) are computed over all
episodes. Continuous trajectory-quality metrics are reported only over
\emph{completed} episodes, because averaging minimum distance, jerk, or final
goal distance over prematurely aborted trajectories can reward a planner for
terminating before entering the interaction. Planning latency is measured
end-to-end per MPC call, including CasADi graph construction, failed primary
attempts, and fallback attempts. We report mean, standard deviation, and
95th-percentile latency.

\subsection{Overall Results}

\begin{table*}[t]
\centering
\caption{Closed-loop results pooled over six scenarios.
Outcome rates are computed over all 300 episodes per method.
Jerk, final goal distance, and minimum distance are averaged over completed
episodes (non-aborted) only. Planning time reports mean/std./95th percentile over all MPC
calls. Lower is better for jerk, goal distance, and time; higher minimum
distance indicates larger executed clearance. Best result is in \textbf{bold}, followed by \underline{underlined} as second best.}
\label{tab:planner_comparison}

\begin{tabular}{lcccccccc}
\toprule
\textbf{Method}
& \textbf{Succ.} $\uparrow$
& \textbf{Viol.} $\downarrow$
& \textbf{Coll.} $\downarrow$
& \textbf{Abort} $\downarrow$
& \textbf{Jerk} $\downarrow$
& \textbf{GoalDist} $\downarrow$
& \textbf{MinDist} $\uparrow$
& \textbf{Plan ms} $\downarrow$ \\
& (\%) & (\%) & (\%) & (\%)
& (m/s$^3$) & (m) & (m) & $\mu/\sigma/p95$ \\
\midrule

Single-mode MPC
& 39.3
& 18.7
& 15.3
& \underline{34.3}
& 4.162
& \textbf{1.270}
& 4.266
& 122.9/156.1/491.6 \\

Risk-Aware BMPC~\cite{zhang2024riskawarebmpc}
& 45.0
& 9.3
& \underline{2.7}
& 44.3
& \textbf{3.763}
& 1.402
& 4.125
& 409.7/487.4/648.0 \\

Game-Theoretic BMPC~\cite{zhang2023bmpcmerge}
& 49.0
& \underline{1.3}
& \textbf{0.0}
& 49.7
& 5.541
& \underline{1.392}
& 4.073
& 318.6/\underline{115.7}/538.0 \\

\midrule

Soft-only
& 93.3
& 6.7
& 3.0
& \textbf{0.0}
& 4.090
& 1.400
& 5.603
& 122.7/237.4/464.4 \\

Hard-only
& \underline{98.7}
& \textbf{0.0}
& \textbf{0.0}
& \textbf{0.0}
& 4.366
& 2.048
& 4.393
& \textbf{99.4}/\textbf{114.0}/\textbf{347.4} \\

Hybrid (no screening)
& 97.0
& \textbf{0.0}
& \textbf{0.0}
& \textbf{0.0}
& \underline{4.025}
& 2.187
& \textbf{6.249}
& 153.0/248.4/708.8 \\

\textbf{Hybrid (Ours)}
& \textbf{99.3}
& \textbf{0.0}
& \textbf{0.0}
& \textbf{0.0}
& 4.035
& 1.491
& \underline{5.871}
& \underline{110.8}/154.1/\underline{360.5} \\

\bottomrule
\end{tabular}
\end{table*}

Table~\ref{tab:planner_comparison} shows that Hybrid attains the highest strict
success rate, \textbf{99.3\%} (298/300 episodes), with no observed collisions
and no recorded rule violations. The corresponding 95\% Wilson interval for
success is $[97.6,99.8]\%$; for 0/300 observed collisions or violations, the
95\% Wilson interval is $[0,1.26]\%$. We therefore report these quantities as
\emph{observed} rates rather than empirical guarantees.

Hard-only and Hybrid share the same explicit hard constraints
  and achieve comparable observed success (98.7\% and 99.3\%),
  with no recorded collisions or violations. Adding soft STL
  shaping increases mean minimum clearance from 4.393 to
  5.871\,m and reduces reported longitudinal jerk from 4.366
  to 4.035\,m/s$^3$. These results support complementary roles
  for hard constraints and soft preferences in the evaluated
  scenarios. Soft-only achieves similar reported jerk
  (4.090\,m/s$^3$), but its strict success falls to 93.3\%,
  with 6.7\% recorded violations and 3.0\% collisions.
  Fig.~\ref{fig:ablation_screening}(a) illustrates its clearance
  violation in one left-turn episode.

The uncertainty-aware baselines exhibit substantially lower strict success.
The single-mode MPC baseline succeeds in 39.3\% of episodes and records both
collisions (15.3\%) and rule violations (18.7\%). Risk-Aware and
Game-Theoretic BMPC reach 45.0\% and 49.0\% success, with abort rates of
44.3\% and 49.7\%. Baseline aborts arise predominantly from a single \textit{Maximum Time Exceeded} event, which terminates the episode; that
the proposed formulation and its hard-only ablation record no aborts---the
latter without any fallback path---indicates that the difference reflects the
numerical conditioning of the constructed program rather than the recovery
mechanism alone. These aborts are important when interpreting trajectory quality,
since pooled averages reward planners that terminate before the interaction;
conditioning on completed episodes removes this bias.
\begin{figure*}[t]
\centering
\includegraphics[width=0.85\textwidth]{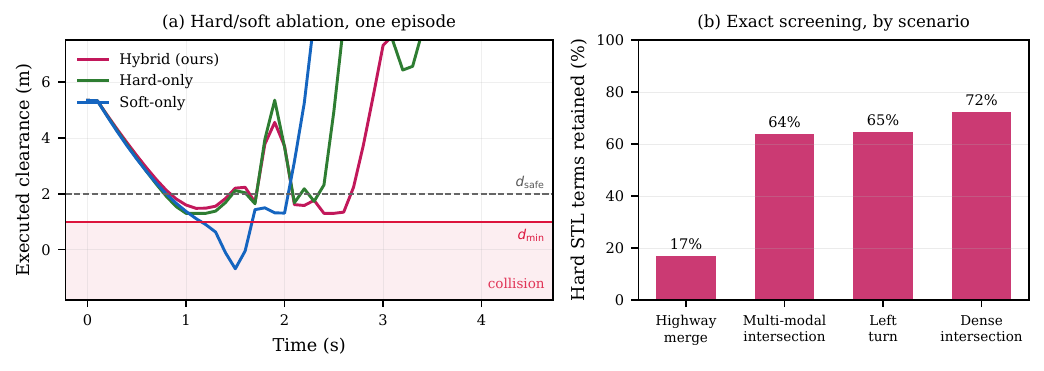}
\caption{Measured effects of both contributions.
(a)~Minimum executed clearance in a challenging left-turn episode. Soft-only
uses the same soft STL costs as Hybrid but no explicit NLP constraints; it crosses
$d_{\min}$ at $t\approx1.15$\,s and collides, whereas Hybrid and Hard-only
remain safe. Across 300 episodes, Soft-only has 6.7\% violations and 3.0\%
collisions, versus 0.0\% for Hybrid (Table~\ref{tab:planner_comparison}).
Hybrid and Hard-only both maintain the hard clearance; aggregate clearance differences appear in Table~\ref{tab:planner_comparison}.
(b)~Fraction of hard-distance terms retained by exact relevance screening,
ranging from 17\% on the highway merge to 72\% at the dense intersection.
By Proposition~\ref{prop:screening}, removed terms are inactive, preserving
the feasible set and objective.}
\label{fig:ablation_screening}
\end{figure*}

\subsection{Per-Scenario Analysis}

\begin{table}[t]
\centering
\caption{Per-scenario strict success / recorded violation rate (\%).
Each entry is ``success / violation'' over 50 trials.}
\label{tab:per_scenario}
\scriptsize
\setlength{\tabcolsep}{2.5pt}
\renewcommand{\arraystretch}{1.08}
\begin{tabular}{l@{\hspace{2pt}}cccccc}
\toprule
\textbf{Scenario}
& \textbf{Hybrid}
& \textbf{Soft}
& \textbf{Hard}
& \textbf{Single-mode}
& \textbf{Risk-BMPC}
& \textbf{Game-BMPC} \\
\midrule
Intersection     & 100/0 & 98/2   & 100/0 & 0/0    & 0/0   & 6/0   \\
Left turn        & 100/0 & 62/38  & 100/0 & 96/4   & 80/2  & 90/8  \\
Highway merge    & 100/0 & 100/0  & 100/0 & 40/60  & 30/54 & 98/0  \\
Multi-modal int. & 100/0 & 100/0  & 100/0 & 0/48   & 94/0  & 0/0   \\
No passing       & 96/0 & 100/0  & 92/0  & 0/0    & 0/0   & 0/0   \\
Stress highway   & 100/0 & 100/0  & 100/0 & 100/0  & 66/0  & 100/0 \\
\bottomrule
\end{tabular}
\vspace{-10pt}
\end{table}

Table~\ref{tab:per_scenario} shows that Hybrid achieves 100\% strict success in
five of the six scenarios and 96\% in the no-passing scenario. The most
diagnostic ablation is the challenging left turn: Soft-only reaches the goal in
all trials, but only 62\% satisfy the strict
definition: 38\% record a rule violation and 18\% collide. Adding the hard safety layer
raises strict success to 100\%. In contrast, removing the soft layer primarily
affects trajectory quality and the no-passing task, where Hard-only reaches
92\% success. The results therefore show different failure modes for the two
components rather than a simple increase in penalty strength.

The multi-modal intersection further highlights the benefit of planning against
all retained behavior hypotheses. Hybrid, Hard-only, and Soft-only all complete
the task in every trial, whereas the single-mode baseline records 48\%
violations and zero strict successes. Game-Theoretic BMPC aborts in every trial of this scenario, whereas Risk-Aware BMPC, which branches explicitly on the supplied modes, completes 94\% at 409.7\,ms mean planning time against 110.8\,ms for Hybrid.

\subsection{Computational Ablation}
\begin{table}[t]
\centering
\caption{Staged latency ablation on the challenging left-turn scenario, paired
over identical seeded episodes. Each row adds one component; the represented
feasible set and objective are unchanged throughout.}
\label{tab:latency_ablation}
\footnotesize
\setlength{\tabcolsep}{6pt}
\begin{tabular}{lrrr}
\toprule
\textbf{Configuration} & \textbf{Build} & \textbf{Solve} & \textbf{Total} \\
                       & (ms)           & (ms)           & (ms)           \\
\midrule
Baseline                    & 234.9 & 484.6 & 733.7 \\
\;+ relevance screening     & 181.3 & 339.0 & 526.3 \\
\;+ vectorized construction & \textbf{22.5} & \textbf{93.0} & \textbf{118.7} \\
\bottomrule
\end{tabular}
\vspace{-15pt}
\end{table}
Screening retains 17--72\% of hard-distance terms in the
  scenarios shown in Fig.~\ref{fig:ablation_screening}(b).
  In Table~\ref{tab:planner_comparison}, it reduces mean planning latency from 153.0 to 110.8\,ms (27.6\%) and
  the pooled 95th percentile from 708.8 to 360.5\,ms. Importantly, screening removes only terms proven
inactive by Eq.~\eqref{eq:screening}; therefore Hybrid and
Hybrid~(no screening) represent the same objective and feasible set. Their
small difference in task success should not be interpreted as a behavioral
effect of screening: reducing graph size changes numerical convergence and timeout incidence under the configured solver time limit.

Table~\ref{tab:latency_ablation} decomposes the two contributions on the
challenging left turn. Of the 615\,ms total reduction, screening accounts for
207\,ms and vectorized construction for 408\,ms. Both reduce solve time as well
as construction time ($484.6\to339.0\to93.0$\,ms), indicating that
expression-graph complexity affects automatic differentiation inside IPOPT and
not only formulation latency. Relative to the unoptimized pre-vectorization implementation, mean planning
time over the full benchmark decreases from 429.4\,ms to 110.8\,ms
(\textbf{74.2\%} reduction).

The resulting planner is therefore substantially faster, but the
110.8\,ms mean and 360.5\,ms 95th-percentile latency do not always constitute a
hard 10\,Hz real-time guarantee. Rather, the results demonstrate that exact
screening and symbolic vectorization make multi-hypothesis STL-MPC
considerably more practical without relaxing the underlying hard constraints.

\textbf{Limitations.}
The vehicle model imposes no explicit lateral-acceleration or curvature
constraint, so we do not claim lateral-comfort guarantees; adding such bounds
trades task success for comfort in the benchmark's tight turns and is left to
future work. Constraint
satisfaction is further conditional on the supplied predictions and holds only
at steps executed from an accepted NLP solution
(Remark~\ref{rem:constraint-satisfaction}). Collision checks model participants as circles (discs) and the ego as a point, so counts are lower bounds.
\section{Conclusion}
\label{sec:conclusion}

We presented a hybrid hard--soft STL-MPC framework that optimizes a single
shared ego trajectory against all retained behavior hypotheses, avoiding
branch-specific ego policies. Minimum-distance and speed requirements enter as
explicit hard constraints, while soft STL robustness terms shape clearance
and comfort within the resulting feasible region. Exact reachability-based relevance screening and vectorized
  symbolic construction reduce computational overhead while
  preserving the represented feasible set and objective. Across six interactive driving scenarios, the
framework attains the highest strict success rate among the compared planners,
with no observed collisions or recorded rule violations under the implemented
execution monitor, and the ablations separate the safety envelope from
behavioral shaping. Future work will address anticipatory speed planning under
an explicit lateral bound, integrate learned interactive predictors with
calibrated uncertainty, and extend finite-horizon constraint satisfaction
toward recursive closed-loop guarantees.

\bibliographystyle{IEEEtran}
\bibliography{references}

\end{document}